\documentclass[times, review, 10pt]{elsarticle}

\usepackage[ruled,vlined,linesnumbered]{algorithm2e}
\usepackage{caption}
\usepackage{subcaption}

\usepackage{makecell}
\usepackage{threeparttable}
\usepackage{siunitx}
\usepackage{algpseudocode}

\usepackage{float}
\usepackage{booktabs}
\usepackage{amssymb}
\usepackage{amsmath}
\usepackage{graphicx}
\usepackage{epsfig}
\usepackage{epstopdf}
\usepackage{hyperref}
\usepackage{multirow}
\usepackage{xcolor}

\usepackage{amsthm}

\newtheorem{theorem}{Theorem}

\usepackage{lineno}

\usepackage{setspace}
\onecolumn

\journal{Elsevier}

\usepackage[T1]{fontenc}
\usepackage[utf8]{inputenc} 
\begin{document}
\newcommand{\rev}[1]{\textcolor{blue}{#1}}
\begin{frontmatter}
\title{MDL-GBG: A Non-parametric and Interpretable Granular-Ball Generation Method for Clustering}

\author[1,2]{Zeqiang Xian}
\ead{xianzeqiang@gnnu.edu.cn}

\author[1,2]{Caihui Liu \corref{cor}}
\ead{liucaihui@gnnu.edu.cn}

\author[1,2]{Yong Zhang}
\ead{zhang_yong@gnnu.edu.cn}

\author[1,2]{Wenjing Qiu}
\ead{wenjingqiu@gnnu.edu.cn}

\author[3]{Duoqian Miao}
\ead{dqmiao@tongji.edu.cn}

\author[4]{Witold Pedrycz}
\ead{wpedrycz@ualberta.ca}

\address[1]{Department of Mathematics and Computer Science, Gannan Normal University, Ganzhou 341000, Jiangxi, China}

\address[2]{Key Laboratory of Data Science and Artificial Intelligence of Jiangxi Education Institutes, Gannan Normal University, Ganzhou 341000, Jiangxi, China}

\affiliation[3]{organization={Department of Computer Science and Technology, Tongji University},
	city={Shanghai},
	postcode={201804},
	country={China}}

\address[4]{Department of Electrical and Computer Engineering, University of Alberta, Edmonton, AB T6R 2V4, Canada}

\cortext[cor]{Corresponding author}

\begin{abstract}
Granular-ball generation requires deciding whether a local region should be retained, split, or partially refined, yet existing methods commonly rely on separate criteria for these decisions. This paper proposes Minimum Description Length-based Granular-Ball Generation (MDL-GBG), a non-parametric and interpretable method that formulates granular-ball generation as local model competition under the Minimum Description Length principle. For each granular ball, three candidate models are considered: a single-ball model, a two-ball model, and a core-ball-residual model. The model with the shortest description length determines whether the current region is retained, split, or partially refined, allowing granular balls to be generated without specifying their number in advance. A residual reassignment mechanism subsequently re-evaluates the samples separated during core-residual refinement after the stable granular balls have been obtained. Experiments on 20 UCI datasets evaluate the generated representation with AC and KMeans++ as downstream clustering procedures. On the 18 datasets with complete results for all eight methods, both MDL-GBG variants obtain higher scores than their respective sample-level counterparts on most datasets. MDL-GBG+AC also obtains the lowest average rank for ARI, ACC, and NMI, with statistically significant differences from several baselines.
\end{abstract}


\end{frontmatter}

\section{Introduction}

Clustering is a fundamental task in unsupervised learning, aiming to uncover the intrinsic organization of unlabeled data~\cite{ding2024survey}. With the rapid growth of real-world datasets in scale, dimensionality, and structural complexity, traditional point-wise clustering methods are increasingly challenged by two closely related factors~\cite{hu2026interpretable,Xia2025_GBCT}. On the one hand, sample-level computation may lead to considerable time and memory costs, especially when pairwise relations or graph structures are involved. On the other hand, noise, heterogeneous densities, weak inter-point connections, and non-convex structures may make local geometric relations unstable and difficult to characterize reliably. These challenges have motivated the exploration of representations beyond individual samples, so that local regions can be described in a compact and structurally meaningful manner.

Granular-ball computing provides a region-level representation in which samples are organized into adaptive balls with explicit centers, radii, and coverage regions. This coarse-to-fine representation is consistent with the global-first cognitive mechanism~\cite{chen1982topological}. Early studies reported that replacing individual samples with granular balls can reduce sample-level computation and improve robustness to local perturbations~\cite{Xia2019_GBClassifiers}. Subsequent work developed granular-ball computing into an adaptive multi-granularity representation and computation paradigm and examined its efficiency, robustness, and interpretability~\cite{Xia2024_GBCSurvey,Xia2024_GBFuzzySet}. Granular balls are therefore used both as compact computational units and as local representations for downstream learning tasks.

This representation has been widely adopted in clustering. Existing studies have incorporated granular balls into spectral clustering~\cite{Xie2023_GBSC,Cheng2024GBUSC}, density peaks clustering~\cite{Cheng2024_GBDP,NI2026GBK_DPC,GUO2026granular_ellipsoid}, DBSCAN-style clustering~\cite{Cheng2024_GBDBSCAN,Luo2025_NaGBDBSCAN}, hierarchical clustering~\cite{JIN2026GBHC}, manifold clustering~\cite{ZHANG2026EGBDPM,jiang2026gb}, and multi-view clustering~\cite{liu2026views,Su2025_MGBCC}. In these methods, granular balls reduce the number of computational units and also participate directly in similarity estimation, graph or density construction, and final cluster formation. The generation stage therefore determines the local units on which these downstream operations are performed.

Accordingly, considerable effort has been devoted to improving granular-ball generation. Some studies reduce repeated partitioning, redesign generation procedures, or incorporate mini-batch density-peak clustering to improve computational efficiency~\cite{Xia2024_EfficientAdaptiveGBG,Xie2024_GBGPP,Yi2026_MiniBatch_GBG}. Others introduce heterogeneous distance measures, local density information, or granularity adjustment to characterize the internal structure of a candidate ball~\cite{Liao2025_ADPGBG,Liu2025_LDGBG,Pan2025_GranularityTuning}. The principle of justifiable granularity has also been incorporated into granular-ball generation and regeneration~\cite{Jia2025_POJGGeneration,Li2025_RGGB}, while more recent studies consider additional structure-aware information~\cite{Wang2026_SAGBC}. Although these approaches refine different aspects of the generation process, retention, splitting, termination, and partial regeneration are still commonly governed by operation-specific criteria. Alternative structural descriptions of the same local region are therefore not generally evaluated under a common decision principle.

The generation process also raises an interpretability question. Granular-ball methods are often regarded as interpretable because the generated balls have explicit geometric meanings, including centers, radii, and coverage regions~\cite{Xia2024_GBCSurvey}. These geometric quantities describe the resulting representation but do not by themselves explain why a local region is retained, divided, or partially refined. The distinction is particularly relevant to clustering because class labels are unavailable to guide local generation decisions. A region may contain separable internal structure, may be partitioned more finely than necessary, or may contain peripheral samples that increase its radius and dispersion.

To address this issue, this paper proposes Minimum Description Length-based Granular-Ball Generation (MDL-GBG) for clustering. The method formulates granular-ball generation as local model competition under the Minimum Description Length principle~\cite{barron1998minimum,grunwald2007minimum}. For each current granular ball, MDL-GBG compares three structural models: the single-ball model, the two-ball model, and the core-ball-residual model. These alternatives represent retention of the current region, division into two sub-balls, and representation by a compact core with peripheral residuals, respectively. The model with the shortest description length determines the subsequent local operation, which places retention, splitting, and core-residual refinement under the same criterion. A residual reassignment procedure is further introduced to reconsider the residual samples after the stable granular balls have been generated.

The main contributions of this work are summarized as follows.

\begin{enumerate}[(1)]
	\item Granular-ball generation is formulated as local model competition under the Minimum Description Length principle, allowing the number and recursive organization of granular balls to be determined from the data rather than specified in advance.
	\item Three local structural models represent retention, binary splitting, and core-residual refinement within a common coding framework.
	\item Description-length comparison provides an explicit basis for each local generation decision, while residual reassignment re-evaluates samples separated during core-residual refinement after the stable balls have been obtained.
	\item Experiments on 20 UCI datasets evaluate the generated representation with AC and KMeans++ as downstream clustering procedures and compare the resulting partitions with sample-level and granular-ball baselines.
\end{enumerate}

The remainder of this paper is organized as follows. Section~\ref{sec:methodology} presents the proposed MDL-GBG method. Section~\ref{sec:complexity} analyzes its computational and space complexity. Section~\ref{sec:experiments} reports the experimental results. Section~\ref{sec:discussion_conclusion} discusses the findings and concludes the paper.

\section{Methodology}
\label{sec:methodology}

This section develops MDL-GBG by expressing granular-ball generation as local model competition under the Minimum Description Length principle. Each current ball is evaluated against alternative structural descriptions, and the resulting code lengths determine whether the region is retained, split, or partially refined.

\subsection{Problem Formulation}

Let $X=\{x_1,x_2,\dots,x_n\}$, where $x_i\in\mathbb{R}^d$, be the input dataset. The goal is to construct a stable granular-ball set $\mathcal{G}^\star=\{B_1,B_2,\dots,B_{m^\star}\},$
together with an optional residual set $\mathcal{R}^\star\subseteq X,$ where $m^\star$ denotes the number of stable granular balls produced by the recursive MDL model-selection process. Each retained ball provides a local representation of its covered samples.

For the sample subset $X_B$ covered by a current granular ball $B$, MDL-GBG considers three candidate local descriptions:
\begin{enumerate}[(1)]
	\item \textbf{Single-ball model} $M_1$, which represents $X_B$ with one compact ball.
	\item \textbf{Two-ball model} $M_2$, which represents $X_B$ with two sub-balls.
	\item \textbf{Core-ball-residual model} $M_3$, which represents $X_B$ with a compact core and a set of peripheral residual samples.
\end{enumerate}

A granular ball is termed stable when the single-ball model is selected and no further local regeneration is applied. For each current granular ball $B$, the optimal local explanation is defined as
\begin{equation}
	M^\star(B)=\mathop{\arg\min}\limits_{M\in\{M_1,M_2,M_3\}}L(M;X_B),
	\label{eq:best_model_en}
\end{equation}
where $L(M;X_B)$ denotes the description length of model $M$ for encoding $X_B$.

\subsection{Local Baseline Assumption}

The single-ball model provides the reference description for a local region and assumes an isotropic Gaussian distribution centered at one local mean.
\begin{equation}
	x_i \overset{i.i.d.}{\sim}\mathcal{N}(c,\sigma^2 I_d),
	\label{eq:single_gaussian_en}
\end{equation}
where $c\in\mathbb{R}^d$ denotes the ball center and $\sigma^2$ is a common dispersion parameter. The isotropic Gaussian model is used as a parsimonious reference model for local description-length comparison.

\subsection{Coarse-Grained Initialization}

Before local MDL competition, the dataset is partitioned into a coarse set of initial granular balls, which are subsequently processed through local model evaluation and regeneration.

For any subset $Y=\{y_1,\dots,y_m\}\subseteq\mathbb{R}^d$, an approximate farthest-point bisection strategy is adopted. The first sample in the fixed ordering of $Y$ is used as the initial anchor, i.e., $y^{(0)}=y_1$. The first farthest point is defined as
\begin{equation}
	y^{(1)}=\arg\max_{y\in Y}\|y-y^{(0)}\|_2,
	\label{eq:init_far1_en}
\end{equation}
and the second farthest point is
\begin{equation}
	y^{(2)}=\arg\max_{y\in Y}\|y-y^{(1)}\|_2.
	\label{eq:init_far2_en}
\end{equation}
The subset is then partitioned according to its relative proximity to the two anchors.
\begin{equation}
	Y_1=\{y\in Y:\|y-y^{(1)}\|_2\le \|y-y^{(2)}\|_2\},
	Y_2=Y\setminus Y_1.
	\label{eq:init_split_en}
\end{equation}

At each initialization step, the current subset with the largest single-ball description length is selected for bisection. The process stops when the target number of initial granular balls is reached or the selected subset contains fewer than $2n_{\min}$ samples. The target number is
\begin{equation}
	k_0=\max\bigl(1,\lfloor\sqrt{n}\rfloor\bigr).
	\label{eq:k0_en}
\end{equation}

The resulting granular balls form the initial regions for the subsequent local MDL competition.

\subsection{Adaptive Minimum Admissible Ball Size}

To ensure sufficient sample support for local model comparison while avoiding a fixed empirical threshold, the minimum admissible ball size is adaptively defined as
\begin{equation}
	n_{\min}
	=
	\max\left(
	2,\;
	\left\lceil
	\min\left(
	\frac{\sqrt{n}}{\log(\sqrt{d+2})},
	\; d+2
	\right)
	\right\rceil
	\right),
	\label{eq:nmin_en}
\end{equation}
where $n$ is the total number of samples and $d$ is the feature dimension.

Equation~\eqref{eq:nmin_en} provides a deterministic, data-dependent admissibility condition for local model comparison. Throughout the paper, $\log$ denotes the natural logarithm.

\subsection{Basic Geometric Quantities}

For any nonempty subset
$
X_B=\{x_1,\dots,x_{n_B}\}\subseteq X,
$
the center of granular ball $B$ is defined as
\begin{equation}
	c_B=\frac{1}{n_B}\sum_{i=1}^{n_B}x_i,
	\label{eq:center_en}
\end{equation}
and its radius is defined as
\begin{equation}
	r_B=\max_{1\le i\le n_B}\|x_i-c_B\|_2.
	\label{eq:radius_en}
\end{equation}

\subsection{Single-Ball Model}

For a sample subset $Y=\{y_1,\dots,y_m\}\subseteq\mathbb{R}^d$, the single-ball model assumes
$
y_i \overset{i.i.d.}{\sim}\mathcal{N}(c,\sigma^2 I_d).
$
Its description length is defined as
\begin{equation}
	L_1(Y)= -\log p(Y\mid \hat c,\hat \sigma^2)+\frac{k_1}{2}\log m.
	\label{eq:L1_en}
\end{equation}
In this expression, the likelihood term reflects the dispersion of $Y$ around a shared local center, while the penalty term accounts for the complexity of the local code. Since the model estimates a $d$-dimensional center and one isotropic variance, $k_1=d+1$. Thus, $L_1(Y)$ measures the description length of representing $Y$ as one coherent granular ball.

\begin{theorem}
	Under the single-ball model, for a subset with positive empirical dispersion, the maximum likelihood estimates are
	\begin{equation}
		\hat c=\bar y=\frac{1}{m}\sum_{i=1}^m y_i,
		\label{eq:mle_mean_en}
	\end{equation}
	and
	\begin{equation}
		\hat\sigma^2=\frac{1}{dm}\sum_{i=1}^m\|y_i-\bar y\|_2^2.
		\label{eq:mle_var_en}
	\end{equation}
\end{theorem}

\begin{proof}
	The joint density can be written as
	\begin{equation}
		p(Y\mid c,\sigma^2)=\prod_{i=1}^{m}\frac{1}{(2\pi\sigma^2)^{d/2}}
		\exp\left(-\frac{\|y_i-c\|_2^2}{2\sigma^2}\right).
	\end{equation}
	Taking the negative logarithm yields
	\begin{equation}
		-\log p(Y\mid c,\sigma^2)
		=
		\frac{md}{2}\log(2\pi\sigma^2)
		+
		\frac{1}{2\sigma^2}\sum_{i=1}^{m}\|y_i-c\|_2^2.
		\label{eq:nll_en}
	\end{equation}
	Minimizing with respect to $c$ is equivalent to minimizing
	$
	\sum_{i=1}^{m}\|y_i-c\|_2^2,
	$
	whose unique minimizer is the sample mean $\bar y$, which gives Eq.~\eqref{eq:mle_mean_en}. Let
	\begin{equation}
		\mathrm{SSE}(Y)=\sum_{i=1}^{m}\|y_i-\bar y\|_2^2.
	\end{equation}
	Substituting $c=\bar y$ into Eq.~\eqref{eq:nll_en} gives
	\begin{equation}
		-\log p(Y\mid \bar y,\sigma^2)
		=
		\frac{md}{2}\log(2\pi\sigma^2)+\frac{\mathrm{SSE}(Y)}{2\sigma^2}.
	\end{equation}
	Differentiating with respect to $\sigma^2$ and setting the derivative to zero yields
	\begin{equation}
		\hat\sigma^2=\frac{\mathrm{SSE}(Y)}{dm},
	\end{equation}
	which is Eq.~\eqref{eq:mle_var_en}.
\end{proof}

Substituting the maximum likelihood estimates into Eq.~\eqref{eq:L1_en} yields the closed form
\begin{equation}
	L_1(Y)
	=
	\frac{md}{2}\Bigl(1+\log(2\pi\hat\sigma^2)\Bigr)
	+\frac{d+1}{2}\log m.
	\label{eq:L1_closed_en}
\end{equation}

In numerical evaluation, $\hat\sigma^2$ is bounded below by $10^{-12}$.

\subsection{Two-Ball Model}

When a current ball contains two local groups, a single-center description can incur a larger dispersion cost than a two-center alternative. The two-ball model therefore evaluates whether the region can be encoded more concisely by two sub-balls.

Suppose that $X_B$ is partitioned into two disjoint subsets $Y_1$ and $Y_2$, satisfying
$
Y_1\cup Y_2=X_B
$
and
$
Y_1\cap Y_2=\varnothing.
$
Let
$
m_j=|Y_j|
$
and
$
\pi_j=m_j/n_B
$
for $j=1,2$. The description length of the two-ball model is defined as
\begin{equation}
	L_2(X_B;Y_1,Y_2)
	=
	L_{\mathrm{part}}(Y_1,Y_2)
	+
	L_1(Y_1)
	+
	L_1(Y_2),
	\label{eq:L2_en}
\end{equation}
where
\begin{equation}
	L_{\mathrm{part}}(Y_1,Y_2)
	=
	n_BH(\pi_1,\pi_2),
	\label{eq:partition_cost_en}
\end{equation}
and
\begin{equation}
	H(\pi_1,\pi_2)
	=
	-\pi_1\log\pi_1
	-\pi_2\log\pi_2.
	\label{eq:entropy_en}
\end{equation}

The first term accounts for the sample assignments induced by the binary partition, whereas $L_1(Y_1)$ and $L_1(Y_2)$ describe the two resulting sub-balls. In particular, the number of binary assignments with subset sizes $m_1$ and $m_2$ is $\binom{n_B}{m_1}$, and the binomial-entropy bounds give
\begin{equation}
	n_BH(\pi_1,\pi_2)-\log(n_B+1)\le\log\binom{n_B}{m_1}\le n_BH(\pi_1,\pi_2).
	\label{eq:entropy_approx_en}
\end{equation}
The entropy-based partition cost therefore differs from the combinatorial cost by at most $\log(n_B+1)$, uniformly over the admissible partition sizes.

To construct candidate partitions, the samples in $X_B$ are projected onto the first principal-component direction and sorted according to their projected values. Each cut position satisfying
$
|Y_1|\ge n_{\min}
$
and
$
|Y_2|\ge n_{\min}
$
defines a feasible candidate.

Let $\mathcal{P}_2(B)$ denote the set of feasible partitions obtained from this ordering. The optimal two-ball description length is defined as
\begin{equation}
	L_2^\star(X_B)
	=
	\min_{(Y_1,Y_2)\in\mathcal{P}_2(B)}
	L_2(X_B;Y_1,Y_2).
	\label{eq:L2_star_en}
\end{equation}
If $\mathcal{P}_2(B)=\varnothing$, the two-ball model is infeasible for the current ball and $L_2^\star(X_B)$ is set to $+\infty$.

The first principal-component projection provides a tractable one-dimensional ordering for candidate generation. The model therefore evaluates all admissible cuts along the direction of largest local variation, rather than searching over all possible bipartitions of $X_B$.

Figure~\ref{fig:m2_two_ball} illustrates the two-ball representation of a local region.

\begin{figure}[h!]
	\centering
	\includegraphics[width=0.3\textwidth]{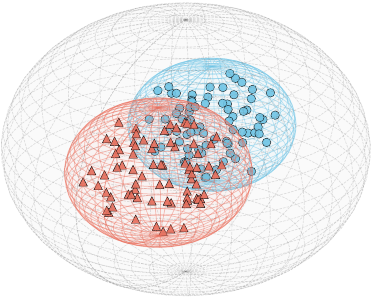}
	\caption{Illustration of the two-ball model $M_2$, in which the current local region is represented by two sub-balls.}
	\label{fig:m2_two_ball}
\end{figure}

\subsection{Core-Ball-Residual Model}

A different local configuration occurs when most samples remain compact around one center while a small number of peripheral samples increase the dispersion of the entire ball. A binary split can separate more of the region than is needed in this case. The core-ball-residual model instead evaluates a compact core while encoding peripheral samples separately.

The samples in $X_B$ are ordered by increasing distance from the ball center $c_B$: 
\begin{equation}
\|x_{(1)}-c_B\|_2
\le
\cdots
\le
\|x_{(n_B)}-c_B\|_2.
\end{equation}

For any residual size $q\in\{1,\dots,n_B-n_{\min}\}$, the corresponding core and residual sets are defined as
\begin{equation}
	C_q=\{x_{(1)},\dots,x_{(n_B-q)}\},
	\qquad
	R_q=\{x_{(n_B-q+1)},\dots,x_{(n_B)}\}.
	\label{eq:core_residual_en}
\end{equation}

Let $r_{C_q}$ and $r_B$ denote the radii of the candidate core and the current ball, respectively. An outer radius is defined by
\begin{equation}
	r_{\mathrm{out}}=\beta\max\{r_B,r_{C_q}\},
	\label{eq:rout_en}
\end{equation}
where $\beta>1$ specifies the spatial extent of the residual region relative to the larger of the current-ball and core-ball radii.

The volume of a $d$-dimensional Euclidean ball with radius $r$ is
\begin{equation}
	V_d(r)=\frac{\pi^{d/2}}{\Gamma(d/2+1)}r^d.
	\label{eq:ball_volume_en}
\end{equation}
Accordingly, the residual reference volume is
\begin{equation}
	V_{\mathrm{shell}}(q)
	=
	V_d(r_{\mathrm{out}})
	-
	V_d(r_{C_q}).
	\label{eq:shell_volume_en}
\end{equation}

The residual subset is selected by the radial ordering in Eq.~\eqref{eq:core_residual_en}, and its cost is evaluated using this reference volume:
\begin{equation}
	L_{\mathrm{res}}(R_q\mid C_q)
	=
	q\log V_{\mathrm{shell}}(q).
	\label{eq:residual_cost_en}
\end{equation}
This term depends on the number of residual samples and the reference volume, with $\log V_{\mathrm{shell}}(q)<0$ when $0<V_{\mathrm{shell}}(q)<1$. At a common measurement precision with cell volume $v_{\delta}=\prod_{\ell=1}^{d}\delta_{\ell}$, the finite-precision approximation of a continuous negative log-density includes $-\log v_{\delta}$ per observation. The three local models describe the same $n_B$ observations, so this shared precision contribution cancels in their pairwise cost comparisons.

The radial ordering provides one candidate core-residual configuration for each residual size $q$. The corresponding index cost is defined as
\begin{equation}
	L_{\mathrm{idx}}(q)
	=
	\log(\max(n_B,2)).
	\label{eq:index_cost_en}
\end{equation}
There are at most $n_B-n_{\min}$ candidate residual sizes, and the index length satisfies $\sum_{q=1}^{n_B-n_{\min}}\exp[-L_{\mathrm{idx}}(q)]\le1$. This identifies the selected residual size among the evaluated candidates.
The description length of the core-ball-residual model is thus
\begin{equation}
	L_3(X_B;q)
	=
	L_1(C_q)
	+
	q\log V_{\mathrm{shell}}(q)
	+
	\log(\max(n_B,2)).
	\label{eq:L3_en}
\end{equation}

All feasible residual sizes are evaluated. Let
\begin{equation}
\mathcal{Q}_B
=
\bigl\{q\in\mathbb{N}:1\le q\le n_B-n_{\min},\;V_{\mathrm{shell}}(q)>0\bigr\}.
\end{equation}

The positivity condition ensures that the logarithmic volume term in Eq.~\eqref{eq:residual_cost_en} is defined. The optimal description length under the core-ball-residual model is
\begin{equation}
	L_3^\star(X_B)
	=
	\min_{q\in\mathcal{Q}_B}
	L_3(X_B;q).
	\label{eq:L3_star_en}
\end{equation}
If $\mathcal{Q}_B=\varnothing$, the core-ball-residual model is infeasible for the current ball and $L_3^\star(X_B)$ is set to $+\infty$.

Figure~\ref{fig:m3_model} illustrates the core-ball-residual model, in which the main local structure is retained as a compact core ball and the peripheral samples are temporarily separated for subsequent reassignment.

\begin{figure}[htbp]
	\centering
	\includegraphics[width=0.3\textwidth]{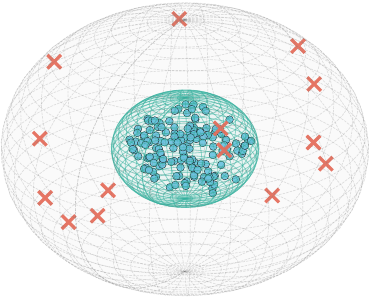}
	\caption{Illustration of the core-ball-residual model $M_3$. A compact core ball represents the main local structure, while peripheral samples are temporarily separated as residuals.}
	\label{fig:m3_model}
\end{figure}

\subsection{Local Model-Selection Criterion}

The single-ball model serves as the reference for deciding whether the current granular ball should be retained or regenerated. The ball $B$ is retained when the description length of $X_B$ under $M_1$ is no greater than the description lengths of the two competing alternatives.
\begin{equation}
	L_1(X_B)
	\le
	\min\{
		L_2^\star(X_B),
		L_3^\star(X_B)
		\}.
	\label{eq:retention_criterion_en}
\end{equation}

The current ball is regenerated when either competing alternative yields a shorter description length.
\begin{equation}
	\min\{
		L_2^\star(X_B),
		L_3^\star(X_B)
		\}
	<
	L_1(X_B).
	\label{eq:regeneration_criterion_en}
\end{equation}
When Eq.~\eqref{eq:regeneration_criterion_en} holds, $B$ is regenerated according to the competing model with the shorter description length. Accordingly, $M_2$ produces two sub-balls, whereas $M_3$ produces a core ball together with a set of residual samples.

\subsection{Local Regeneration and Residual Reassignment}

If $M^\star(B)=M_1$, the current ball is added to the stable set. If $M^\star(B)=M_2$, it is regenerated into two sub-balls, which are appended to the processing queue. If $M^\star(B)=M_3$, the core ball is returned to the processing queue, while the residual samples are appended to a residual queue in their radial order.

After all stable balls have been obtained, the residual samples are processed sequentially according to their order in the residual queue. For a residual sample $x$ and a stable ball $B_j$, the attachment cost is defined as
\begin{equation}
	\Delta_j(x)
	=
	L_1\bigl(X_{B_j}\cup\{x\}\bigr)
	-
	L_1(X_{B_j}),
	\label{eq:attach_cost_en}
\end{equation}
which measures the change in description length when $x$ is incorporated into $B_j$. Alternatively, $x$ can be encoded by the background model, with cost
\begin{equation}
	\Delta_{\mathrm{bg}}(x)=\log V_{\mathrm{bg}},
	\label{eq:bg_cost_en}
\end{equation}
where $V_{\mathrm{bg}}>0$ denotes the volume of the global bounding region of the input data. A uniform index over the $m^\star$ stable balls and the background option has length $\log(m^\star+1)$ for every destination. This term cancels in the comparison of their incremental costs. Let
\begin{equation}
	j^\star=\arg\min_j\Delta_j(x).
\end{equation}
The sample $x$ is assigned to $B_{j^\star}$ when
\begin{equation}
	\Delta_{j^\star}(x)<\Delta_{\mathrm{bg}}(x);
\end{equation}
otherwise, it is retained in the residual set. Once a sample is assigned, the statistics of the corresponding stable ball are updated, and the remaining residual samples are evaluated against the updated granular-ball representation. Ties are resolved according to the fixed ordering of the stable balls. After all residual samples have been processed, any unassigned samples are retained in $\mathcal{R}^\star$ and treated as one residual ball in the downstream representation.

\subsection{Overall Procedure}

The overall workflow of MDL-GBG is illustrated in Fig.~\ref{fig:framework} and summarized in Algorithm~\ref{alg:mdlgbg}. A coarse initial granular-ball set is first constructed by successive approximate farthest-point bisections. Each ball is then evaluated under the single-ball, two-ball, and core-ball-residual models. The ball is retained when $M_1$ gives the shortest description length; otherwise, it is regenerated according to $M_2$ or $M_3$. Samples separated by $M_3$ are collected in a residual pool and reconsidered after the stable granular balls have been obtained.

The granular balls are processed through a first-in-first-out queue. Each accepted regeneration produces smaller subsets, while $n_{\min}$ excludes inadmissibly small descendants. The procedure therefore terminates after a finite number of local model comparisons.

\begin{figure}[!t]
	\centering
	\includegraphics[width=\textwidth]{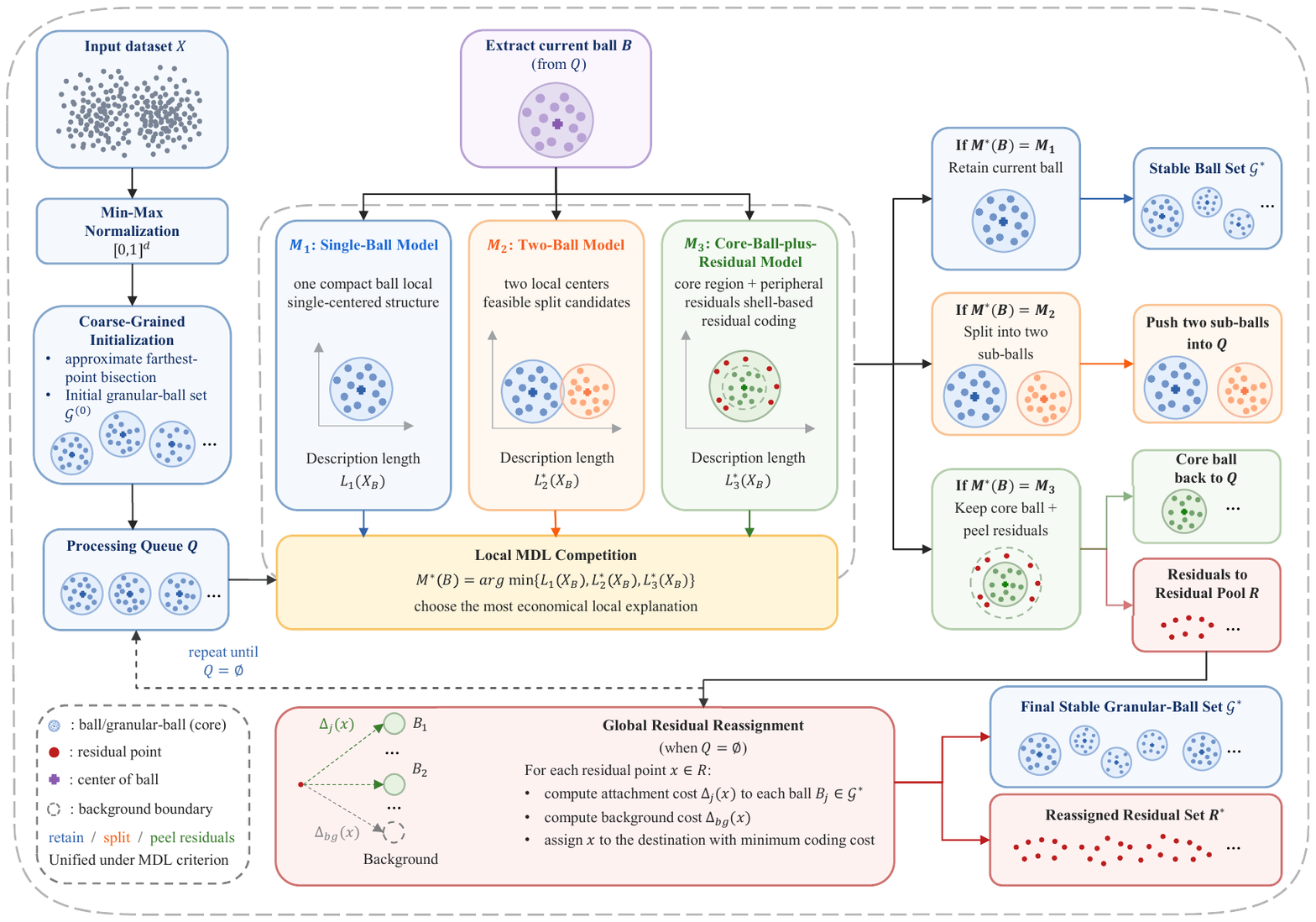}
	\caption{Overall framework of MDL-GBG. The generation process is driven by local MDL competition among the single-ball, two-ball, and core-ball-residual models, followed by residual reassignment after stable granular balls are obtained.}
	\label{fig:framework}
\end{figure}

\begin{algorithm}[!t]
	\caption{MDL-GBG}
	\label{alg:mdlgbg}
	\KwIn{Dataset $X\subseteq\mathbb{R}^d$}
	\KwOut{Final stable granular-ball set $\mathcal{G}^\star$ and remaining residual set $\mathcal{R}^\star$}
	
	\textbf{Initialization} \\
	Generate the coarse-grained initial ball set $\mathcal{G}^{(0)}$ by repeatedly bisecting the subset with the largest single-ball cost, until $k_0$ balls are obtained or the selected subset is too small to bisect \;
	$\mathcal{Q}\gets \mathcal{G}^{(0)}$, 
	$\mathcal{G}^\star\gets \emptyset$, 
	$\mathcal{R}\gets [\,]$, $\mathcal{R}^{\star}\gets \emptyset$ \;
	
	\While{$\mathcal{Q}\neq \emptyset$}{
		Extract a current granular ball $B$ from $\mathcal{Q}$ \;
		Compute $L_1(X_B)$ \;
		Enumerate all feasible candidate splits in $\mathcal{P}_2(B)$ and obtain $L_2^\star(X_B)$ \;
		Enumerate all feasible residual sizes in $\mathcal{Q}_B$ and obtain $L_3^\star(X_B)$ \;
		Select $M^\star(B)$ according to Eq.~\eqref{eq:best_model_en} using $L_1(X_B)$, $L_2^\star(X_B)$, and $L_3^\star(X_B)$ \;
		
		\eIf{$M^\star(B)=M_1$}{
			Add $B$ to $\mathcal{G}^\star$ \;
		}{
			\eIf{$M^\star(B)=M_2$}{
				Regenerate $B$ into two sub-balls $B_1,B_2$ \;
				Push $B_1$ and $B_2$ into $\mathcal{Q}$ \;
			}{
				Regenerate $B$ into one core ball $B_c$ and one residual set $R_B$ \;
				Push $B_c$ into $\mathcal{Q}$ \;
				Append the samples in $R_B$ to the tail of $\mathcal{R}$ in radial order \;
			}
		}
	}
	
	\ForEach{residual sample $x$ in the order stored in $\mathcal{R}$}{
		Compute the attachment cost $\Delta_j(x)$ for each stable granular ball $B_j$ \;
		Compute the background cost $\Delta_{\mathrm{bg}}(x)$ \;
		\eIf{$\min_j\Delta_j(x)<\Delta_{\mathrm{bg}}(x)$}{
			Assign $x$ to the stable ball with the smallest attachment cost \;
			Update the statistics of the selected stable ball \;
		}{
			Retain $x$ in the residual-background set $\mathcal{R}^{\star}$ \;
		}
	}
	
	\Return{$\mathcal{G}^\star,\mathcal{R}^\star$}
	
\end{algorithm}

\section{Computational Complexity Analysis}
\label{sec:complexity}

Let $n$ denote the number of samples and $d$ the feature dimension. For a current granular ball $B$, let $n_B$ be the number of samples contained in $B$. The cost of evaluating the single-ball model is $\mathcal{O}(n_Bd)$, since only the local mean, squared residuals, and the corresponding coding cost need to be computed.

For the two-ball model, candidate generation starts from the first principal-component direction of the samples in $B$. Computing this direction costs $\mathcal{O}(n_Bd^2)$ when a covariance-based implementation is used. Projecting the samples and sorting the projected values require $\mathcal{O}(n_Bd+n_B\log n_B)$. Once the sorted order is obtained, prefix sums of coordinates and squared norms are used to evaluate the coding costs of candidate sub-balls, so feasible cuts can be evaluated without repeatedly scanning the two induced subsets. If $S_B$ denotes the number of feasible cut positions, then $S_B\le n_B$, and the cost of the two-ball model is
\begin{equation}
	T_{M_2}(B)
	=
	\mathcal{O}\bigl(n_Bd^2+n_Bd+n_B\log n_B+S_Bd\bigr)
	=
	\mathcal{O}\bigl(n_Bd^2+n_Bd+n_B\log n_B\bigr).
	\label{eq:complexity_m2_short}
\end{equation}

For the core-ball-residual model, the samples are first sorted according to their distances from the current ball center, which costs $\mathcal{O}(n_Bd+n_B\log n_B)$. Candidate cores are then obtained by progressively peeling peripheral samples. Prefix sufficient statistics allow the likelihood-related part of the single-ball coding cost to be updated efficiently for each candidate core. However, the radius of each candidate core depends on its updated center and sample composition, and therefore still requires candidate-level geometric evaluation. If $R_B$ denotes the number of feasible residual sizes, with $R_B\le n_B$, the cost of this model is
\begin{equation}
	T_{M_3}(B)
	=
	\mathcal{O}\bigl(n_Bd+n_B\log n_B+R_Bn_Bd\bigr)
	=
	\mathcal{O}\bigl(n_B^2d+n_B\log n_B\bigr).
	\label{eq:complexity_m3_short}
\end{equation}

Combining the above components, the local MDL evaluation cost for one current granular ball is
\begin{equation}
	T_B
	=
	\mathcal{O}\bigl(n_Bd^2+n_B^2d+n_B\log n_B\bigr).
	\label{eq:local_complexity_short}
\end{equation}

At the global level, the coarse-grained initialization based on approximate farthest-point bisection has a conservative upper bound of $\mathcal{O}(n^{3/2}d)$. During local regeneration, different granular balls are processed sequentially. If $n_{B_1},n_{B_2},\ldots$ denote the sizes of all processed balls, the total refinement cost can be written as
\begin{equation}
	\sum_B
	\mathcal{O}\bigl(n_Bd^2+n_B^2d+n_B\log n_B\bigr).
	\label{eq:refinement_sum_complexity}
\end{equation}

The total refinement cost depends on how sample subsets are distributed across the regeneration process. When regeneration is approximately balanced, the processed balls are distributed over logarithmically many levels, and summing Eq.~\eqref{eq:refinement_sum_complexity} over these levels gives
\begin{equation}
	T_{\mathrm{balanced}}
	=
	\mathcal{O}\bigl(nd^2\log n+n^2d+n\log^2 n\bigr).
	\label{eq:balanced_complexity_short}
\end{equation}
In contrast, if regeneration is highly unbalanced, a sequence of large balls may be evaluated before termination. A conservative upper bound in this case is
\begin{equation}
	T_{\mathrm{worst}}
	=
	\mathcal{O}\bigl(n^2d^2+n^3d+n^2\log n\bigr).
	\label{eq:worst_complexity_short}
\end{equation}

After the stable granular balls have been generated, residual reassignment compares each residual point with all stable balls. With the sufficient statistics of each stable ball maintained during reassignment, one attachment-cost update requires $\mathcal{O}(d)$ time. If $R$ residual points and $m^\star$ stable granular balls are obtained, this step costs $\mathcal{O}(Rm^\star d)$, which is bounded by $\mathcal{O}(n^2d)$. Assigning the $n$ observations to their nearest final ball centers requires $\mathcal{O}(nm_f d)$ time, where $m_f$ is the number of final ball units. If an additional downstream clustering algorithm is applied to the stable-ball centers, its cost depends on the selected clustering backend and is separate from the granular-ball generation and nearest-center assignment costs.

The space complexity of MDL-GBG is
\begin{equation}
	S=\mathcal{O}(nd+d^2),
	\label{eq:total_space_complexity}
\end{equation}
where $\mathcal{O}(nd)$ is used for storing the data, ball memberships, queues, residual pools, and local working arrays, while $\mathcal{O}(d^2)$ is used for principal-direction computation. The dominant computational cost arises from repeated local model evaluation. Among the three candidates, $M_3$ has the highest local worst-case cost because each feasible residual size induces a different candidate core and residual shell.

\section{Experiments}
\label{sec:experiments}

\subsection{Experimental Settings}

Experiments were conducted on 20 UCI datasets\footnote{\url{https://archive.ics.uci.edu/}} with different sample sizes, feature dimensions, and numbers of classes, as summarized in Table~\ref{tab:dataset_info}. Before clustering, each nonconstant feature was rescaled to $[0,1]$ by Min--Max normalization.
\begin{equation}
	x'_{ij}
	=
	\frac{x_{ij}-\min(x_{\cdot j})}
	{\max(x_{\cdot j})-\min(x_{\cdot j})}.
	\label{eq:minmax_en}
\end{equation}
Constant features were mapped to zero. The same normalized data were supplied to all compared methods.

The granular balls generated by MDL-GBG were evaluated with Agglomerative Clustering (AC) and K-means using k-means++ initialization (KMeans++)~\cite{arthur2007k}. When $\mathcal{R}^\star$ was nonempty, it was treated as one residual ball in the final representation. Let $m_f$ denote the total number of final granular-ball units after this step and $K_c$ the prescribed number of clusters. When $m_f>K_c$, the ball centers were grouped into $K_c$ clusters. When $m_f\le K_c$, the granular-ball identifiers were directly used as the ball-level clustering result. The resulting ball-level labels were mapped to samples through nearest-center assignment using the final granular-ball centers. Ground-truth labels were used only to determine $K_c$ and to calculate the evaluation metrics.

The comparison included the sample-level AC and KMeans++ baselines together with four granular-ball clustering methods: GBCT\footnote{\url{https://github.com/wylbdthxbw/GBC}}~\cite{Xia2025_GBCT}, GBSC\footnote{\url{https://github.com/xjnine/GBSC}}~\cite{Xie2023_GBSC}, WGBC\footnote{\url{https://github.com/xjnine/W-GBC}}~\cite{xie2024_WGBC}, and MDMSC\footnote{\url{https://github.com/SWJTU-ML/MDMSC}}~\cite{xu2025_MDMSC}. Methods involving random initialization or stochastic optimization were evaluated over 20 independent runs with different random seeds, and their results are reported as mean $\pm$ standard deviation. Deterministic methods were executed once. The ablation and parameter-sensitivity experiments used the deterministic MDL-GBG+AC pipeline.

AC and KMeans++ were implemented using scikit-learn~\cite{pedregosa2011scikit}, while the official implementations were used for the remaining comparison methods. Recommended or default settings were adopted without dataset-specific parameter tuning. For MDL-GBG, $k_0$ and $n_{\min}$ were determined by Eqs.~\eqref{eq:k0_en} and \eqref{eq:nmin_en}, respectively. The shell-radius factor was fixed at $\beta=2$ before the comparative benchmark experiments and was kept unchanged across datasets. Its influence is examined in Section~\ref{subsec:shell_sensitivity}.

All experiments were conducted in Python 3.12 under Windows 11 on a system equipped with an AMD Ryzen 9 8940HX CPU, an NVIDIA GeForce RTX 5070Ti GPU with 12 GB VRAM, and 32 GB RAM. The source code is publicly available at \url{https://github.com/AnonymousUser0816/MDL-GBG}.

\begin{table}[!t]
	\centering
	\caption{Dataset information and abbreviations.}
	\label{tab:dataset_info}
	\scriptsize
	\setlength{\tabcolsep}{4pt}
	\renewcommand{\arraystretch}{0.95}
	\begin{tabular}{llrrr}
		\toprule
		Abbr. & Dataset & Samples & Features & Classes \\
		\midrule
		Zoo & Zoo & 101 & 16 & 7 \\
		Iris & Iris & 150 & 4 & 3 \\
		Wine & Wine & 178 & 13 & 3 \\
		Glass & Glass Identification & 214 & 9 & 6 \\
		Heart & Statlog (Heart) & 270 & 13 & 2 \\
		Ecoli & Ecoli & 336 & 7 & 8 \\
		Australian & Statlog (Australian Credit Approval) & 690 & 14 & 2 \\
		Segment & Statlog (Image Segmentation) & 2310 & 19 & 7 \\
		Churn & Iranian Churn & 3150 & 13 & 2 \\
		Rice & Rice (Cammeo and Osmancik) & 3810 & 7 & 2 \\
		Waveform & Waveform Database Generator (Version 1) & 5000 & 21 & 3 \\
		Digits & Optical Recognition of Handwritten Digits & 5620 & 64 & 10 \\
		Landsat & Statlog (Landsat Satellite) & 6435 & 36 & 6 \\
		Isolet & Isolet & 7797 & 617 & 26 \\
		PenDigits & Pen-Based Recognition of Handwritten Digits & 10992 & 16 & 10 \\
		Bean & Dry Bean & 13611 & 16 & 7 \\
		HTRU2 & HTRU2 & 17898 & 8 & 2 \\
		Shuttle & Statlog (Shuttle) & 58000 & 7 & 7 \\
		CDC & CDC Diabetes Health Indicators & 253680 & 21 & 2 \\
		Crop & Crop Mapping Using Fused Optical-Radar Data & 325834 & 174 & 7 \\
		\bottomrule
	\end{tabular}
\end{table}

\subsection{Evaluation Metrics}

The clustering performance was evaluated using the adjusted Rand index (ARI), clustering accuracy (ACC), and normalized mutual information (NMI). These metrics characterize the agreement between the predicted partition and the ground-truth classes from complementary perspectives.

Let
$
\mathbf{y}=\{y_1,\dots,y_n\}
$
and
$
\hat{\mathbf{y}}=\{\hat y_1,\dots,\hat y_n\}
$
denote the ground-truth and predicted labels, respectively. ACC measures the proportion of correctly assigned samples after resolving the permutation ambiguity of cluster labels.
\begin{equation}
	\mathrm{ACC}
	=
	\frac{1}{n}
	\max_{\phi}
	\sum_{i=1}^{n}
	\mathbb{I}\bigl(y_i=\phi(\hat y_i)\bigr),
	\label{eq:acc_en}
\end{equation}
where $\phi$ denotes a one-to-one mapping between predicted cluster labels and ground-truth classes, obtained using the Hungarian algorithm.

ARI evaluates the agreement between two partitions based on sample pairs while correcting for agreement expected by chance. Let $n_{uv}$ denote the number of samples shared by the $u$-th ground-truth class and the $v$-th predicted cluster, with
$
a_u=\sum_v n_{uv}
$
and
$
b_v=\sum_u n_{uv}.
$
ARI is defined as
\begin{equation}
	\mathrm{ARI}
	=
	\frac{
		\displaystyle
		\sum_{u,v}\binom{n_{uv}}{2}
		-
		\frac{
			\displaystyle
			\sum_u\binom{a_u}{2}
			\sum_v\binom{b_v}{2}
		}{
			\binom{n}{2}
		}
	}{
		\displaystyle
		\frac{1}{2}
		\left[
		\sum_u\binom{a_u}{2}
		+
		\sum_v\binom{b_v}{2}
		\right]
		-
		\frac{
			\displaystyle
			\sum_u\binom{a_u}{2}
			\sum_v\binom{b_v}{2}
		}{
			\binom{n}{2}
		}
	}.
	\label{eq:ari_en}
\end{equation}

NMI measures the information shared by the two partitions and is defined as
\begin{equation}
	\mathrm{NMI}
	=
	\frac{
		2I(\mathbf{y};\hat{\mathbf{y}})
	}{
		H(\mathbf{y})+H(\hat{\mathbf{y}})
	},
	\label{eq:nmi_en}
\end{equation}
where $I(\mathbf{y};\hat{\mathbf{y}})$ is their mutual information and $H(\mathbf{y})$ and $H(\hat{\mathbf{y}})$ are the corresponding entropies.

Higher values of ARI, ACC, and NMI indicate greater agreement with the reference partition.

\subsection{Overall Clustering Performance}

Table~\ref{tab:all_metrics_comparison} reports ARI, ACC, and NMI for the 20 UCI datasets, together with the average values computed from the available numerical results. Rank-based statistical comparisons on the common dataset set are reported in Section~\ref{subsec:statistical_significance}.

\begin{table}[!t]
	\centering
	\caption{Dataset-level comparison of ARI, ACC, and NMI on the 20 UCI datasets.}
	\label{tab:all_metrics_comparison}
	\begin{threeparttable}
		\scriptsize
		\renewcommand{\arraystretch}{0.8}
		\resizebox{\textwidth}{!}{
			\begin{tabular}{llcccccccc}
				\toprule
				\multirow{2}{*}{Dataset} & \multirow{2}{*}{Metric} & \multicolumn{2}{c}{Ours} & \multicolumn{2}{c}{Baseline} & \multicolumn{4}{c}{Granular-Ball Methods} \\
				\cmidrule(lr){3-4}
				\cmidrule(lr){5-6}
				\cmidrule(lr){7-10}
				& & \makecell{MDL-GBG\\+AC} & \makecell{MDL-GBG\\+KM++} & AC & KM++ & \makecell{WGBC\\ (ICDE 2024)} & \makecell{MDMSC\\ (AAAI 2025)} & \makecell{GBSC\\ (TKDE 2023)} & \makecell{GBCT\\ (TNNLS 2025)} \\
				\midrule
				\multirow{3}{*}{Zoo}
				& ARI & 0.8277 & 0.8151 $\pm$ 0.0563 & 0.6159 & 0.6625 $\pm$ 0.0913 & \textbf{0.9403 $\pm$ 0.0228} & 0.5442 & 0.1927 $\pm$ 0.1441 & 0.5333 \\
				& ACC & 0.8614 & 0.8530 $\pm$ 0.0376 & 0.7525 & 0.7560 $\pm$ 0.0614 & \textbf{0.9069 $\pm$ 0.0188} & 0.6634 & 0.4346 $\pm$ 0.0724 & 0.6238 \\
				& NMI & 0.8789 & 0.8738 $\pm$ 0.0227 & 0.7403 & 0.7591 $\pm$ 0.0385 & \textbf{0.9026 $\pm$ 0.0124} & 0.7515 & 0.4255 $\pm$ 0.1068 & 0.6026 \\
				\midrule
				\multirow{3}{*}{Iris}
				& ARI & 0.6956 & 0.6809 $\pm$ 0.0366 & 0.7592 & 0.7219 $\pm$ 0.0070 & 0.7727 $\pm$ 0.0227 & \textbf{0.9038} & 0.6537 $\pm$ 0.0000 & 0.5681 \\
				& ACC & 0.8733 & 0.8646 $\pm$ 0.0224 & 0.9067 & 0.8893 $\pm$ 0.0033 & 0.9133 $\pm$ 0.0106 & \textbf{0.9667} & 0.8467 $\pm$ 0.0000 & 0.6667 \\
				& NMI & 0.7709 & 0.7516 $\pm$ 0.0430 & 0.8057 & 0.7484 $\pm$ 0.0082 & 0.7852 $\pm$ 0.0090 & \textbf{0.8851} & 0.7490 $\pm$ 0.0000 & 0.7337 \\
				\midrule
				\multirow{3}{*}{Wine}
				& ARI & \textbf{0.9472} & 0.8431 $\pm$ 0.1430 & 0.2926 & 0.3628 $\pm$ 0.0153 & 0.8732 $\pm$ 0.0147 & 0.7414 & 0.5522 $\pm$ 0.0000 & 0.7033 \\
				& ACC & \textbf{0.9831} & 0.9385 $\pm$ 0.0833 & 0.6124 & 0.6643 $\pm$ 0.0595 & 0.9567 $\pm$ 0.0055 & 0.9101 & 0.8146 $\pm$ 0.0000 & 0.8933 \\
				& NMI & \textbf{0.9275} & 0.8372 $\pm$ 0.1174 & 0.4049 & 0.4257 $\pm$ 0.0063 & 0.8569 $\pm$ 0.0114 & 0.7528 & 0.6430 $\pm$ 0.0000 & 0.7397 \\
				\midrule
				\multirow{3}{*}{Glass}
				& ARI & \textbf{0.2937} & 0.2210 $\pm$ 0.0314 & 0.0198 & 0.2584 $\pm$ 0.0157 & 0.2084 $\pm$ 0.0213 & 0.1651 & 0.1426 $\pm$ 0.0003 & 0.0396 \\
				& ACC & 0.5234 & 0.4836 $\pm$ 0.0221 & 0.3785 & \textbf{0.5346 $\pm$ 0.0194} & 0.4771 $\pm$ 0.0171 & 0.4112 & 0.4206 $\pm$ 0.0000 & 0.4065 \\
				& NMI & \textbf{0.4638} & 0.3888 $\pm$ 0.0408 & 0.1145 & 0.4030 $\pm$ 0.0224 & 0.3664 $\pm$ 0.0246 & 0.3595 & 0.2822 $\pm$ 0.0014 & 0.1608 \\
				\midrule
				\multirow{3}{*}{Heart}
				& ARI & \textbf{0.3757} & 0.3107 $\pm$ 0.0734 & -0.0010 & 0.0283 $\pm$ 0.0012 & 0.1668 $\pm$ 0.0123 & 0.1025 & 0.0284 $\pm$ 0.0000 & 0.0241 \\
				& ACC & \textbf{0.8074} & 0.7759 $\pm$ 0.0485 & 0.5519 & 0.5898 $\pm$ 0.0016 & 0.7063 $\pm$ 0.0069 & 0.6630 & 0.5926 $\pm$ 0.0000 & 0.5926 \\
				& NMI & \textbf{0.2917} & 0.2435 $\pm$ 0.0565 & 0.0002 & 0.0188 $\pm$ 0.0007 & 0.1294 $\pm$ 0.0144 & 0.1032 & 0.0723 $\pm$ 0.0000 & 0.0395 \\
				\midrule
				\multirow{3}{*}{Ecoli}
				& ARI & 0.6866 & 0.5052 $\pm$ 0.0728 & \textbf{0.7449} & 0.4289 $\pm$ 0.0492 & 0.6991 $\pm$ 0.0542 & 0.4930 & 0.3211 $\pm$ 0.0374 & 0.5205 \\
				& ACC & \textbf{0.7917} & 0.6374 $\pm$ 0.0630 & 0.7649 & 0.5632 $\pm$ 0.0439 & 0.7641 $\pm$ 0.0449 & 0.6339 & 0.5036 $\pm$ 0.0203 & 0.6607 \\
				& NMI & 0.6773 & 0.6126 $\pm$ 0.0203 & \textbf{0.7193} & 0.6034 $\pm$ 0.0264 & 0.6580 $\pm$ 0.0293 & 0.6218 & 0.4765 $\pm$ 0.0353 & 0.5696 \\
				\midrule
				\multirow{3}{*}{Australian}
				& ARI & 0.4277 & 0.3414 $\pm$ 0.1031 & 0.0007 & 0.0031 $\pm$ 0.0008 & \textbf{0.4802 $\pm$ 0.0353} & 0.0215 & 0.0859 $\pm$ 0.0815 & -0.0030 \\
				& ACC & 0.8275 & 0.7882 $\pm$ 0.0517 & 0.5565 & 0.5610 $\pm$ 0.0015 & \textbf{0.8465 $\pm$ 0.0135} & 0.5826 & 0.6338 $\pm$ 0.0708 & 0.5348 \\
				& NMI & 0.3533 & 0.2754 $\pm$ 0.0909 & 0.0034 & 0.0134 $\pm$ 0.0032 & \textbf{0.4049 $\pm$ 0.0337} & 0.0133 & 0.1115 $\pm$ 0.0813 & 0.0676 \\
				\midrule
				\multirow{3}{*}{Segment}
				& ARI & 0.4735 & \textbf{0.4888 $\pm$ 0.0424} & 0.0001 & 0.3381 $\pm$ 0.0576 & 0.0000 $\pm$ 0.0000 & 0.0811 & 0.4754 $\pm$ 0.0142 & 0.3747 \\
				& ACC & 0.5814 & \textbf{0.6261 $\pm$ 0.0532} & 0.1459 & 0.5027 $\pm$ 0.0589 & 0.1429 $\pm$ 0.0000 & 0.2584 & 0.6119 $\pm$ 0.0187 & 0.5346 \\
				& NMI & \textbf{0.6776} & 0.6190 $\pm$ 0.0262 & 0.0153 & 0.5094 $\pm$ 0.0478 & 0.0000 $\pm$ 0.0000 & 0.3338 & 0.6379 $\pm$ 0.0199 & 0.5732 \\
				\midrule
				\multirow{3}{*}{Churn}
				& ARI & -0.0209 & 0.3663 $\pm$ 0.0263 & -0.0834 & -0.1074 $\pm$ 0.0000 & \textbf{0.3941 $\pm$ 0.0348} & 0.0086 & 0.3577 $\pm$ 0.0000 & 0.3377 \\
				& ACC & 0.8289 & 0.8364 $\pm$ 0.0213 & 0.7622 & 0.6680 $\pm$ 0.0007 & 0.8645 $\pm$ 0.0327 & 0.4987 & 0.8295 $\pm$ 0.0000 & \textbf{0.8832} \\
				& NMI & 0.0095 & 0.2260 $\pm$ 0.0286 & 0.0403 & 0.0738 $\pm$ 0.0003 & \textbf{0.2482 $\pm$ 0.0304} & 0.0064 & 0.2167 $\pm$ 0.0000 & 0.2322 \\
				\midrule
				\multirow{3}{*}{Rice}
				& ARI & 0.5808 & \textbf{0.6881 $\pm$ 0.0020} & 0.5413 & 0.5776 $\pm$ 0.0004 & 0.6755 $\pm$ 0.0000 & 0.6753 & 0.6824 $\pm$ 0.0000 & -0.0021 \\
				& ACC & 0.8814 & \textbf{0.9148 $\pm$ 0.0006} & 0.8682 & 0.8802 $\pm$ 0.0001 & 0.9110 $\pm$ 0.0000 & 0.9110 & 0.9131 $\pm$ 0.0000 & 0.5677 \\
				& NMI & 0.5002 & \textbf{0.5748 $\pm$ 0.0023} & 0.4633 & 0.4690 $\pm$ 0.0004 & 0.5632 $\pm$ 0.0000 & 0.5709 & 0.5694 $\pm$ 0.0000 & 0.0052 \\
				\midrule
				\multirow{3}{*}{Waveform}
				& ARI & \textbf{0.4209} & 0.3566 $\pm$ 0.0921 & 0.2980 & 0.2651 $\pm$ 0.0513 & 0.3770 $\pm$ 0.0932 & 0.2512 & 0.2518 $\pm$ 0.0000 & 0.1181 \\
				& ACC & \textbf{0.7418} & 0.6496 $\pm$ 0.1001 & 0.5574 & 0.5154 $\pm$ 0.0613 & 0.6687 $\pm$ 0.1252 & 0.5224 & 0.5150 $\pm$ 0.0000 & 0.5110 \\
				& NMI & \textbf{0.4529} & 0.4211 $\pm$ 0.0449 & 0.3878 & 0.3687 $\pm$ 0.0292 & 0.4193 $\pm$ 0.0419 & 0.3592 & 0.3702 $\pm$ 0.0000 & 0.2312 \\
				\midrule
				\multirow{3}{*}{Digits}
				& ARI & 0.6326 & 0.6511 $\pm$ 0.0502 & 0.5703 & 0.6297 $\pm$ 0.0442 & 0.3191 $\pm$ 0.0170 & \textbf{0.8333} & 0.2876 $\pm$ 0.0017 & 0.0003 \\
				& ACC & 0.7418 & 0.7600 $\pm$ 0.0566 & 0.6349 & 0.7493 $\pm$ 0.0476 & 0.4979 $\pm$ 0.0267 & \textbf{0.8776} & 0.4376 $\pm$ 0.0030 & 0.1181 \\
				& NMI & 0.7738 & 0.7438 $\pm$ 0.0277 & 0.7617 & 0.7332 $\pm$ 0.0211 & 0.4824 $\pm$ 0.0154 & \textbf{0.9029} & 0.4597 $\pm$ 0.0044 & 0.0426 \\
				\midrule
				\multirow{3}{*}{Landsat}
				& ARI & 0.3583 & 0.5264 $\pm$ 0.0111 & 0.2734 & 0.4883 $\pm$ 0.0739 & \textbf{0.5608 $\pm$ 0.0171} & 0.5599 & 0.4819 $\pm$ 0.0027 & 0.0005 \\
				& ACC & 0.6117 & 0.6732 $\pm$ 0.0104 & 0.4923 & 0.6505 $\pm$ 0.0604 & \textbf{0.7247 $\pm$ 0.0176} & 0.7120 & 0.7157 $\pm$ 0.0013 & 0.2415 \\
				& NMI & 0.5115 & 0.6090 $\pm$ 0.0028 & 0.4264 & 0.5837 $\pm$ 0.0501 & 0.6100 $\pm$ 0.0103 & \textbf{0.6567} & 0.5748 $\pm$ 0.0027 & 0.0100 \\
				\midrule
				\multirow{3}{*}{Isolet}
				& ARI & 0.4471 & 0.4585 $\pm$ 0.0236 & 0.1460 & \textbf{0.4751 $\pm$ 0.0268} & 0.1922 $\pm$ 0.0069 & 0.4599 & 0.1560 $\pm$ 0.0132 & 0.0078 \\
				& ACC & 0.5153 & 0.5303 $\pm$ 0.0237 & 0.2138 & 0.5425 $\pm$ 0.0273 & 0.3065 $\pm$ 0.0076 & \textbf{0.5639} & 0.2454 $\pm$ 0.0072 & 0.1075 \\
				& NMI & 0.7315 & 0.7194 $\pm$ 0.0135 & 0.5799 & 0.7224 $\pm$ 0.0093 & 0.4671 $\pm$ 0.0086 & \textbf{0.7600} & 0.4236 $\pm$ 0.0113 & 0.2159 \\
				\midrule
				\multirow{3}{*}{PenDigits}
				& ARI & 0.4410 & \textbf{0.5720 $\pm$ 0.0320} & 0.4139 & 0.5609 $\pm$ 0.0281 & 0.0000 $\pm$ 0.0000 & 0.4107 & 0.4247 $\pm$ 0.0002 & 0.0014 \\
				& ACC & 0.5977 & \textbf{0.7166 $\pm$ 0.0368} & 0.5548 & 0.7063 $\pm$ 0.0455 & 0.1041 $\pm$ 0.0000 & 0.5759 & 0.5741 $\pm$ 0.0001 & 0.1398 \\
				& NMI & 0.6617 & 0.6995 $\pm$ 0.0144 & 0.6517 & 0.6827 $\pm$ 0.0084 & 0.0000 $\pm$ 0.0000 & \textbf{0.7001} & 0.5970 $\pm$ 0.0004 & 0.0775 \\
				\midrule
				\multirow{3}{*}{Bean}
				& ARI & 0.5324 & 0.6079 $\pm$ 0.0256 & 0.2487 & 0.3800 $\pm$ 0.0122 & 0.0000 $\pm$ 0.0000 & 0.6303 & \textbf{0.6679 $\pm$ 0.0000} & 0.0978 \\
				& ACC & 0.6545 & 0.7355 $\pm$ 0.0362 & 0.3978 & 0.5572 $\pm$ 0.0151 & 0.2605 $\pm$ 0.0000 & 0.7598 & \textbf{0.7969 $\pm$ 0.0000} & 0.3377 \\
				& NMI & 0.6983 & 0.6965 $\pm$ 0.0097 & 0.4511 & 0.5150 $\pm$ 0.0031 & 0.0000 $\pm$ 0.0000 & \textbf{0.7329} & 0.7274 $\pm$ 0.0000 & 0.1868 \\
				\midrule
				\multirow{3}{*}{HTRU2}
				& ARI & \textbf{0.6201} & 0.5233 $\pm$ 0.0000 & -0.0006 & -0.0780 $\pm$ 0.0002 & 0.3376 $\pm$ 0.0004 & 0.5272 & 0.6049 $\pm$ 0.0002 & 0.0324 \\
				& ACC & \textbf{0.9549} & 0.9171 $\pm$ 0.0000 & 0.9076 & 0.7659 $\pm$ 0.0011 & 0.9044 $\pm$ 0.0001 & 0.9161 & 0.9383 $\pm$ 0.0000 & 0.9103 \\
				& NMI & \textbf{0.4943} & 0.3299 $\pm$ 0.0000 & 0.0000 & 0.0267 $\pm$ 0.0002 & 0.1528 $\pm$ 0.0003 & 0.3406 & 0.3971 $\pm$ 0.0002 & 0.0277 \\
				\midrule
				\multirow{3}{*}{Shuttle}
				& ARI & \textbf{0.5695} & 0.1843 $\pm$ 0.0406 & 0.0007 & 0.3802 $\pm$ 0.1351 & 0.0000 $\pm$ 0.0000 & -0.0098 & 0.1852 $\pm$ 0.0000 & 0.4534 \\
				& ACC & 0.7603 & 0.4246 $\pm$ 0.0655 & 0.7859 & \textbf{0.8047 $\pm$ 0.0279} & 0.7860 $\pm$ 0.0000 & 0.3077 & 0.7330 $\pm$ 0.0000 & 0.7662 \\
				& NMI & \textbf{0.5800} & 0.3878 $\pm$ 0.0244 & 0.0007 & 0.2084 $\pm$ 0.0754 & 0.0000 $\pm$ 0.0000 & 0.0086 & 0.3749 $\pm$ 0.0000 & 0.4784 \\
				\midrule
				\multirow{3}{*}{CDC}
				& ARI & \textbf{0.1550} & 0.0825 $\pm$ 0.0169 & ME & 0.1138 $\pm$ 0.0064 & 0.0846 $\pm$ 0.0139 & ME & ME & 0.0032 \\
				& ACC & 0.8024 & 0.6515 $\pm$ 0.0231 & ME & 0.7971 $\pm$ 0.0086 & 0.6675 $\pm$ 0.0324 & ME & ME & \textbf{0.8548} \\
				& NMI & 0.0443 & \textbf{0.0762 $\pm$ 0.0149} & ME & 0.0255 $\pm$ 0.0022 & 0.0701 $\pm$ 0.0161 & ME & ME & 0.0002 \\
				\midrule
				\multirow{3}{*}{Crop}
				& ARI & \textbf{0.7017} & 0.5737 $\pm$ 0.0228 & ME & 0.1298 $\pm$ 0.0186 & 0.0000 $\pm$ 0.0000 & ME & ME & 0.0657 \\
				& ACC & \textbf{0.7718} & 0.6486 $\pm$ 0.0193 & ME & 0.3482 $\pm$ 0.0129 & 0.2611 $\pm$ 0.0000 & ME & ME & 0.3209 \\
				& NMI & \textbf{0.7581} & 0.6704 $\pm$ 0.0085 & ME & 0.2258 $\pm$ 0.0258 & 0.0000 $\pm$ 0.0000 & ME & ME & 0.1373 \\
				\midrule
				\multirow{3}{*}{Average}
				& ARI & 0.5083 & 0.4898 $\pm$ 0.0451 & 0.2689 & 0.3310 $\pm$ 0.0318 & 0.3541 $\pm$ 0.0183 & 0.4111 & 0.3640 $\pm$ 0.0164 & 0.1938 \\
				& ACC & 0.7556 & 0.7213 $\pm$ 0.0388 & 0.6025 & 0.6523 $\pm$ 0.0279 & 0.6335 $\pm$ 0.0180 & 0.6519 & 0.6421 $\pm$ 0.0108 & 0.5336 \\
				& NMI & 0.5629 & 0.5378 $\pm$ 0.0305 & 0.3648 & 0.4058 $\pm$ 0.0190 & 0.3558 $\pm$ 0.0129 & 0.4922 & 0.4505 $\pm$ 0.0146 & 0.2566 \\
				\bottomrule
		\end{tabular}}
		\begin{tablenotes}[flushleft]
			\footnotesize
			\item[] \parbox{\textwidth}{%
				ME denotes MemoryError. Average values are computed over the available datasets for each method. For stochastic methods, the $\pm$ term in the Average row denotes the mean of the dataset-wise standard deviations.%
			}
		\end{tablenotes}
	\end{threeparttable}
\end{table}

The available-case averages in Table~\ref{tab:all_metrics_comparison} summarize the numerical results reported by each method. MDL-GBG+AC obtains average ARI, ACC, and NMI values of 0.5083, 0.7556, and 0.5629, respectively, while the corresponding values for MDL-GBG+KM++ are 0.4898, 0.7213, and 0.5378. AC, MDMSC, and GBSC have no numerical results on CDC and Crop because of MemoryError, so their averages are based on 18 datasets rather than 20. Cross-method significance testing is therefore performed on the 18 datasets completed by all methods.

Each MDL-GBG variant can be compared directly with its sample-level counterpart because the downstream procedure is unchanged. Across the 18 common datasets, MDL-GBG+AC has higher ARI and ACC on 16 datasets and higher NMI on 15 datasets than AC. MDL-GBG+KM++ has higher values than KMeans++ on 14 datasets for ARI and ACC and on 16 datasets for NMI. This pattern is not uniform across datasets. AC retains higher ARI on Iris and Ecoli, while KMeans++ has higher ARI and ACC on Iris, Glass, Isolet, and Shuttle. The two MDL-GBG variants also differ despite sharing the same granular-ball representation. Across all 20 datasets, MDL-GBG+AC has higher ARI and ACC than MDL-GBG+KM++ on 12 datasets and higher NMI on 15 datasets, while the reverse occurs on 8, 8, and 5 datasets, respectively. The final partition therefore depends on both the generated representation and the downstream clustering criterion.

The comparison with existing granular-ball methods shows that the relative differences vary across baselines. MDL-GBG+AC obtains higher ARI, ACC, and NMI than GBCT on 17, 16, and 17 of the 18 common datasets, respectively, and the corresponding counts against GBSC are 13, 13, and 14. The comparison with WGBC and MDMSC is closer, with MDL-GBG+AC obtaining higher values on approximately two-thirds of the common datasets. Several competing methods nevertheless achieve higher results on particular datasets, including WGBC on Zoo and Australian, MDMSC on Iris and Digits, and GBSC on Bean. These dataset-level variations motivate the rank-based comparison in the following subsection, where all methods are evaluated over the same 18 datasets.

\subsection{Statistical Significance Analysis}
\label{subsec:statistical_significance}

Statistical differences among the methods are examined with the Friedman test~\citep{Friedman1940} followed by the Nemenyi post-hoc test~\citep{demvsar2006}. Because larger ARI, ACC, and NMI values indicate better clustering results, the methods are ranked in descending order on each dataset and then averaged across datasets. Smaller average ranks correspond to better overall performance. CDC and Crop contain MemoryError entries for several methods, so the tests are conducted on the remaining 18 datasets with complete numerical results for all methods. For stochastic methods, the dataset-wise mean score over the 20 runs is used to determine the rank on each dataset. The significance level is set to 0.05.

\begin{figure}[!t]
	\centering
	\begin{subfigure}{0.49\textwidth}
		\centering
		\includegraphics[width=\linewidth]{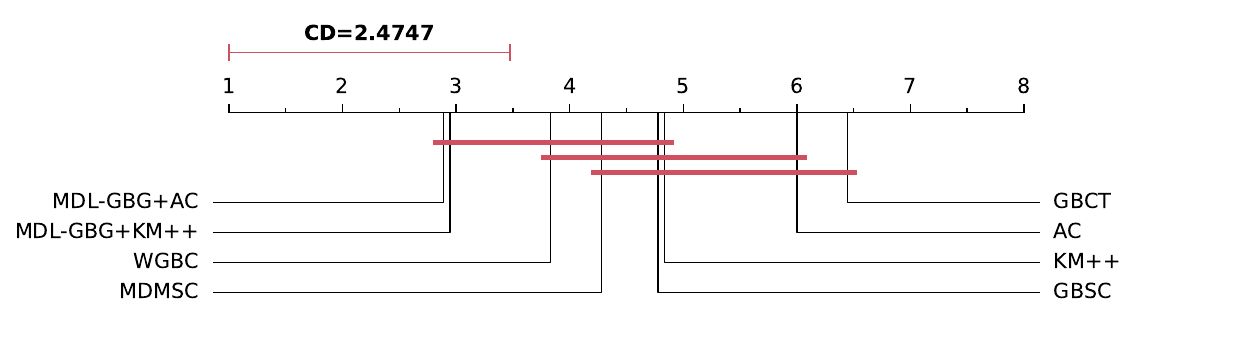}
		\caption{ARI}
	\end{subfigure}
	\begin{subfigure}{0.49\textwidth}
		\centering
		\includegraphics[width=\linewidth]{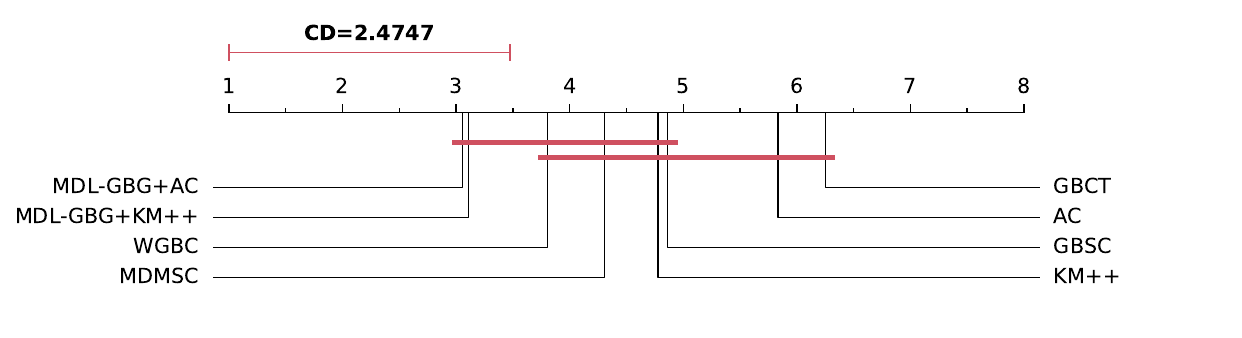}
		\caption{ACC}
	\end{subfigure}
	\begin{subfigure}{0.5\textwidth}
		\centering
		\includegraphics[width=\linewidth]{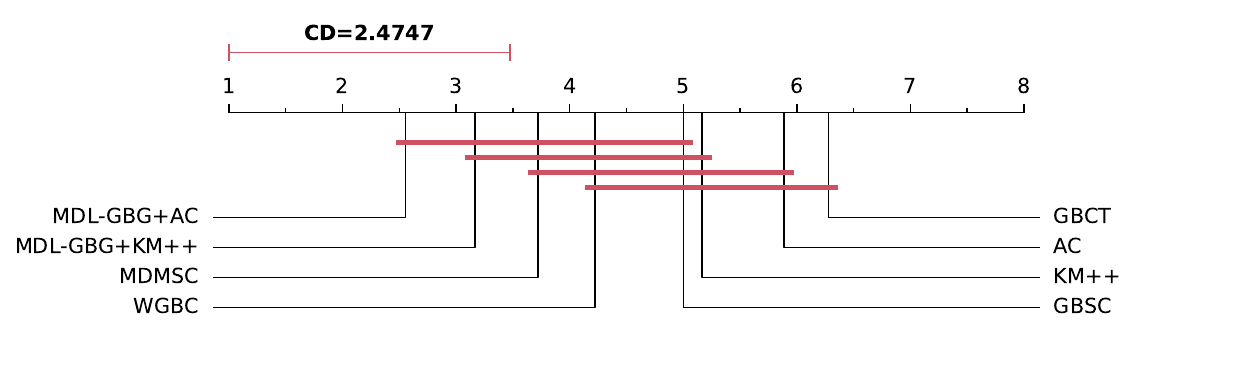}
		\caption{NMI}
	\end{subfigure}
	\caption{Critical difference diagrams of the Nemenyi post-hoc test on ARI, ACC, and NMI. Methods connected by a horizontal line are not significantly different at the 0.05 level.}
	\label{fig:cd_diagram}
\end{figure}

The Friedman test rejects the null hypothesis of equal performance ranks across the methods, with $p=1.0\times 10^{-5}$, $p=1.58\times 10^{-4}$, and $p=7.0\times 10^{-6}$ for ARI, ACC, and NMI, respectively. For eight methods evaluated on 18 common datasets, the Nemenyi critical difference at the 0.05 level is 2.4747. The corresponding critical difference diagrams are shown in Fig.~\ref{fig:cd_diagram}.

MDL-GBG+AC obtains the lowest average rank among the eight methods for all three metrics, with ranks of 2.89 for ARI, 3.06 for ACC, and 2.56 for NMI. MDL-GBG+KM++ follows with average ranks of 2.94, 3.11, and 3.17. The two variants are close in average rank for ARI and ACC, while their separation is larger for NMI.

The Nemenyi comparisons identify the pairwise rank differences that exceed the critical difference on the 18 common datasets. For ARI and ACC, this occurs between MDL-GBG+AC and AC and between MDL-GBG+AC and GBCT. For NMI, the corresponding differences with KMeans++, AC, and GBCT exceed the critical difference. The differences between MDL-GBG+AC and WGBC, MDMSC, and GBSC remain below the critical difference for all three metrics.

\subsection{Ablation Study}
\label{subsec:ablation}

To assess the contribution of two core components of MDL-GBG, an ablation study was conducted on the same 20 UCI datasets used in the main comparison. The components examined were the core-ball-residual model $M_3$ and the residual reassignment mechanism. The model $M_3$ allows a local region to be represented by a compact core and a set of peripheral residual samples, while residual reassignment revisits these samples after the stable granular balls have been generated. AC was used throughout as the downstream clustering method, and the same experimental settings were maintained for all variants.

The complete MDL-GBG+AC method was compared with two ablated variants. In the first variant, $M_3$ was removed from the local model competition, leaving single-ball retention and binary splitting as the available local decisions. In the second variant, $M_3$ was retained but residual reassignment was disabled. The residual subsets produced during local regeneration were therefore preserved as independent balls rather than being reassigned to the stable granular balls. Preprocessing, initialization, and the downstream AC procedure were kept unchanged. A variant removing both components is not reported separately because disabling $M_3$ produces no residual pool and consequently renders residual reassignment inactive.

\begin{table}[!t]
	\centering
	\caption{Ablation results of the MDL-GBG+AC variants on the 20 UCI datasets. "Balls" denotes the average number of final granular balls, while "Res. Gen." and "Res. Reas." denote the average numbers of generated and reassigned residual samples, respectively.}
	\label{tab:ablation}
	\scriptsize
	\setlength{\tabcolsep}{4pt}
	\renewcommand{\arraystretch}{1.05}
	\begin{tabular}{lcccccc}
		\toprule
		Variant & ARI & ACC & NMI & Balls & Res. Gen. & Res. Reas. \\
		\midrule
		Full MDL-GBG+AC & \textbf{0.5083} & \textbf{0.7556} & \textbf{0.5629} & 866.4 & 2035.3 & 1880.8 \\
		w/o $M_3$ & 0.4657 & 0.7416 & 0.5267 & 941.1 & 0.0 & 0.0 \\
		w/o residual reassignment & 0.3417 & 0.6678 & 0.4153 & 1357.5 & 2035.3 & 0.0 \\
		\bottomrule
	\end{tabular}
\end{table}

Removing $M_3$ changes the average ARI, ACC, and NMI from 0.5083, 0.7556, and 0.5629 to 0.4657, 0.7416, and 0.5267, respectively, while the average number of final granular balls increases from 866.4 to 941.1. In this variant, local model competition is restricted to retaining the current ball or dividing it into two sub-balls. The structural change is therefore introduced during regeneration, before any residual reassignment is performed.

Disabling residual reassignment changes the same averages to 0.3417, 0.6678, and 0.4153 and increases the average number of final granular balls to 1357.5. The number of generated residual samples remains 2035.3 because the regeneration stage is unchanged. In the complete method, 1880.8 residual samples are reassigned on average, corresponding to approximately 92.4\% of the generated residuals. Most samples separated during core-residual regeneration are therefore reconsidered and incorporated into stable balls rather than remaining as separate final structures. The two ablations affect different stages of the method, with $M_3$ determining the structural alternatives available during regeneration and residual reassignment determining how residual samples are incorporated after the stable balls have been formed.

\subsection{Sensitivity Analysis}
\label{subsec:shell_sensitivity}

The core-ball-residual model represents peripheral samples through a shell region whose outer radius is determined by the factor $\beta$ in Eq.~\eqref{eq:rout_en}. Since $\beta$ controls the spatial extent of the residual code, its value may affect both the selection of the core-ball-residual explanation and the resulting granular-ball structure. The sensitivity analysis was therefore conducted using the deterministic MDL-GBG+AC pipeline, so that the configurations differ only in the value of $\beta$ and are not affected by random variation in downstream clustering. The tested values were $\{1.25,1.50,1.75,2.00,2.25,2.50,3.00\}$. Seven datasets with different sample sizes, feature dimensions, and numbers of classes were included, namely Iris, Wine, Ecoli, Glass Identification, Statlog (Image Segmentation), Optical Recognition of Handwritten Digits, and Dry Bean.

Table~\ref{tab:shell_sensitivity_avg} reports the average ARI, ACC, and NMI over the seven datasets. The three averages rise as $\beta$ increases from 1.25 to 2.00 and reach 0.6088, 0.7356, and 0.7128 at $\beta=2.00$. They decrease at $\beta=2.25$ and remain below the $\beta=2.00$ values at $\beta=2.50$ and $3.00$.

\begin{table}[t]
	\centering
	\caption{Average sensitivity results of MDL-GBG+AC under different shell-radius factors on seven datasets.}
	\label{tab:shell_sensitivity_avg}
	\begin{tabular}{c|ccc}
		\toprule
		$\beta$ & ARI & ACC & NMI \\
		\midrule
		1.25 & 0.5705 & 0.7079 & 0.6690 \\
		1.50 & 0.5754 & 0.7100 & 0.6788 \\
		1.75 & 0.5976 & 0.7262 & 0.6989 \\
		\textbf{2.00} & \textbf{0.6088} & \textbf{0.7356} & \textbf{0.7128} \\
		2.25 & 0.5336 & 0.6857 & 0.6552 \\
		2.50 & 0.5502 & 0.6911 & 0.6712 \\
		3.00 & 0.5470 & 0.6838 & 0.6692 \\
		\bottomrule
	\end{tabular}
\end{table}

Figure~\ref{fig:shell_sensitivity_acc} shows that the value of $\beta$ yielding the highest ACC varies across datasets, although the seven-dataset average is highest at $\beta=2.00$.

\begin{figure}[t]
	\centering
	\includegraphics[width=0.78\linewidth]{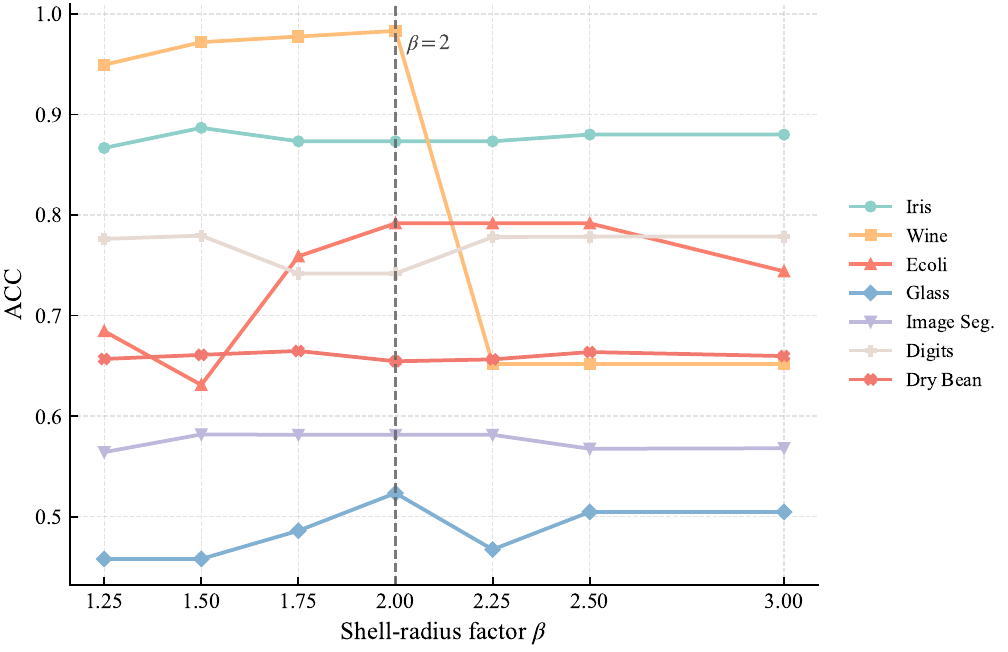}
	\caption{Dataset-level ACC of MDL-GBG+AC under different shell-radius factors $\beta$.}
	\label{fig:shell_sensitivity_acc}
\end{figure}

The structural statistics in Fig.~\ref{fig:shell_sensitivity_structure} decrease overall as $\beta$ increases. A larger outer radius enlarges the shell volume in Eq.~\eqref{eq:L3_en}, which increases the coding cost assigned to residual samples and reduces the frequency with which $M_3$ is selected. The clustering metrics do not follow the same monotonic pattern and reach their maxima at different $\beta$ values across datasets. The frequency of core-residual regeneration is therefore one component of the resulting representation rather than a direct surrogate for clustering quality.

\begin{figure}[t]
	\centering
	\includegraphics[width=0.78\linewidth]{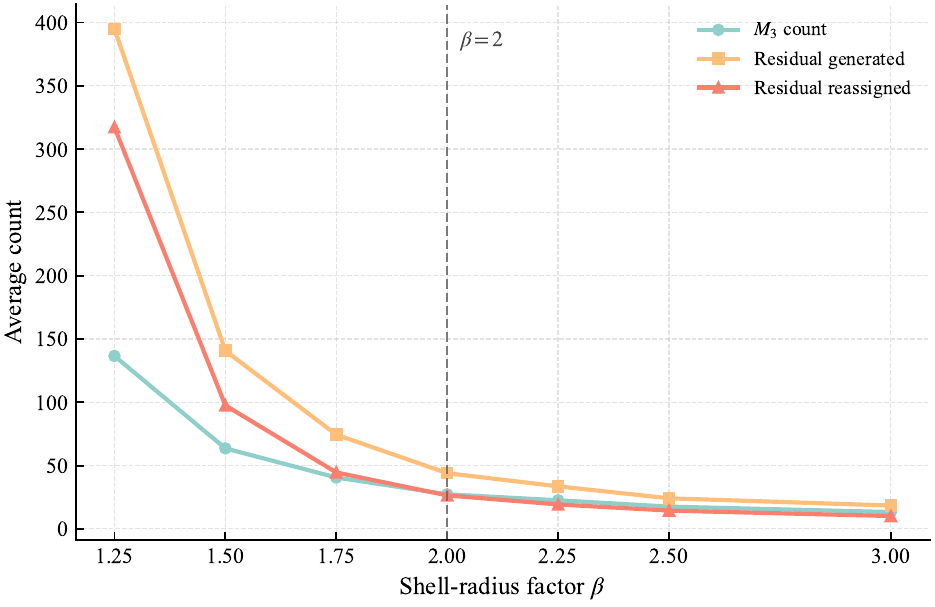}
	\caption{Average structural statistics of MDL-GBG under different shell-radius factors $\beta$.}
	\label{fig:shell_sensitivity_structure}
\end{figure}

\subsection{Runtime Analysis}

The runtime experiment examines the computational cost of MDL-based granular-ball generation. Table~\ref{tab:runtime_grouped} reports the average runtime for three sample-size groups, and Fig.~\ref{fig:runtime_bubble} shows the runtime of MDL-GBG+AC across datasets with different sample sizes and feature dimensionalities. Since MDL-GBG+AC and MDL-GBG+KM++ show similar runtime trends, only MDL-GBG+AC is included in the bubble plot.

\begin{table}[htbp]
	\centering
	\caption{Grouped average runtime (seconds) across different sample-scale regimes.}
	\label{tab:runtime_grouped}
	\begin{threeparttable}
		\scriptsize
		\resizebox{\textwidth}{!}{%
			\begin{tabular}{lcccccccc}
				\toprule
				Sample scale & MDL-GBG+AC & MDL-GBG+KM++ & AC & KM++ & WGBC & MDMSC & GBSC & GBCT \\
				\midrule
				$\leq 5000$ & 0.2574 & 0.5401 & \textbf{0.0817} & 0.3515 & 1.6277 & 8.0995 & 1.2328 & 13.8824 \\
				$5000 < n \leq 50000$ & 23.0661 & 21.6747 & 3.1032 & \textbf{0.4065} & 27.4582 & 27.7924 & 41.5725 & 65.2882 \\
				$>50000$ & 27834.4844 & 27972.6602 & 176.6339\tnote{a} & \textbf{1.6171} & 225.8146 & 438.9400\tnote{a} & 5427.3152\tnote{a} & 1287.4568 \\
				\bottomrule
		\end{tabular}}
		
		\vspace{0.3em}
		\begin{minipage}{\textwidth}
			\footnotesize
			\tnote{a} The average is computed over executable datasets only. For AC, MDMSC, and GBSC, the results on CDC and Crop are unavailable due to MemoryError.
		\end{minipage}
	\end{threeparttable}
\end{table}

The grouped runtimes in Table~\ref{tab:runtime_grouped} increase with sample scale. MDL-GBG+AC requires 0.2574\,s, 23.0661\,s, and 27834.4844\,s on average across the three groups, and MDL-GBG+KM++ follows a similar pattern. For datasets with at most 50,000 samples, both MDL-GBG variants have lower grouped average runtimes than WGBC, MDMSC, GBSC, and GBCT. Their relative runtimes against sample-level AC and KMeans++ differ between the two sample-size groups. In the largest group, the average runtime of the MDL-GBG variants is approximately $2.8\times10^4$ seconds; the averages for AC, MDMSC, and GBSC are based on the datasets completed without MemoryError. This increase is consistent with the local complexity of $M_3$, whose repeated candidate evaluation becomes more costly as the number of samples in a current ball grows.

\begin{figure}[htbp]
	\centering
	\includegraphics[width=0.7\textwidth]{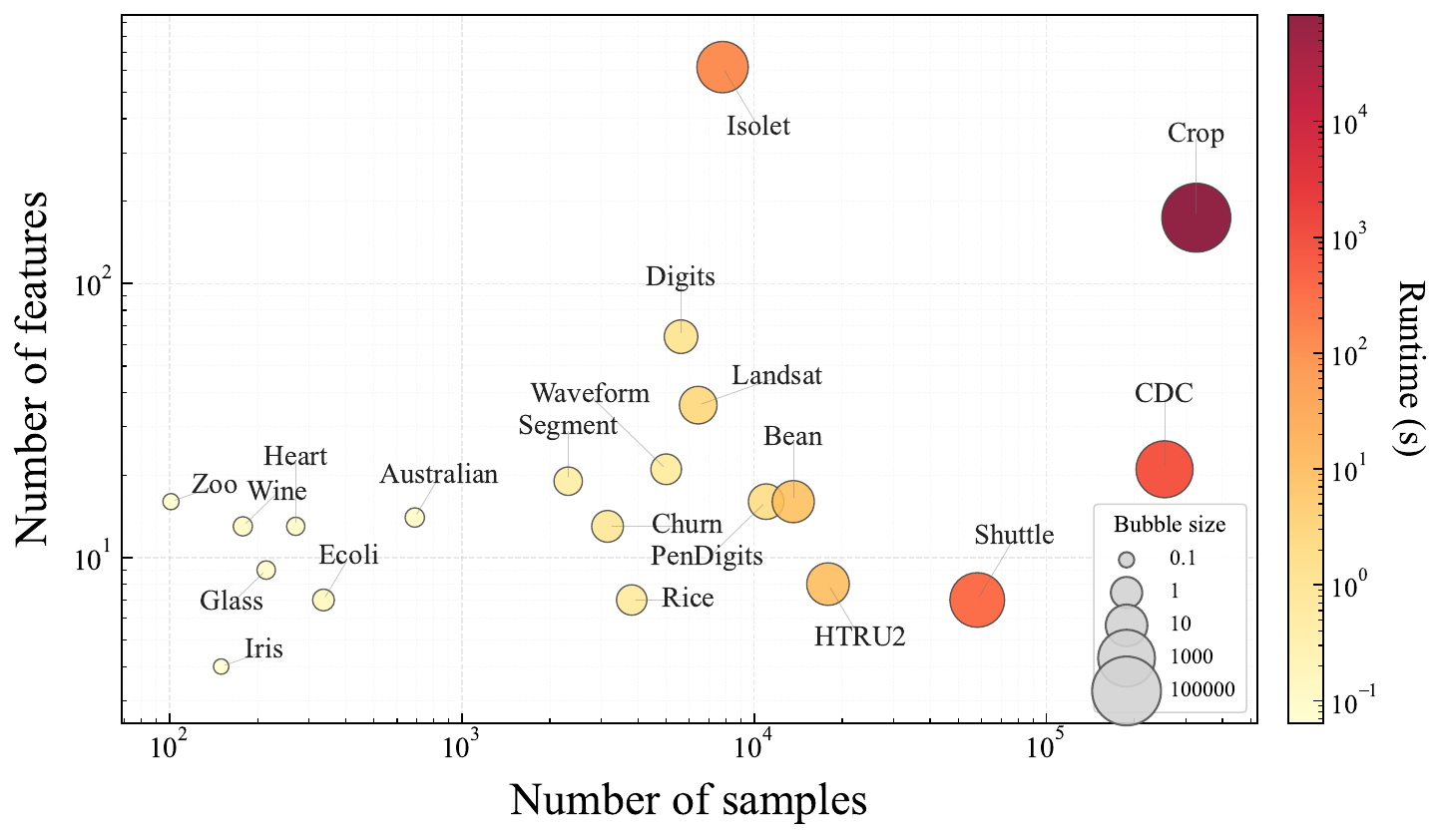}
	\caption{Runtime of MDL-GBG+AC across datasets with different sample sizes and feature dimensionalities.}
	\label{fig:runtime_bubble}
\end{figure}

The time and memory measurements exhibit different scaling behavior. MDL-GBG has space complexity $\mathcal{O}(nd+d^2)$ and does not construct a full sample-to-sample matrix. Under the 32~GB configuration used in the experiments, it completes on CDC and Crop, whereas AC, MDMSC, and GBSC report MemoryError on these datasets. The two MDL-GBG variants nevertheless require approximately $2.8\times10^4$ seconds on average in the same large-scale group. Memory remains within the available hardware configuration for these datasets, while execution time becomes the more restrictive factor.

\section{Discussion and Conclusions}
\label{sec:discussion_conclusion}

This study formulates granular-ball generation as local model competition under the Minimum Description Length principle. MDL-GBG compares three alternative descriptions of each local region, corresponding to retaining a single ball, dividing it into two sub-balls, or separating a compact core from peripheral samples. The description-length comparisons determine whether a region is retained or refined, with successive local decisions establishing the number and recursive organization of the resulting granular balls. After stable balls have been obtained, residual samples are reconsidered for reassignment according to their incremental description costs. Those for which a stable ball offers a lower cost than the background model are reassigned, while the remaining samples are retained as residuals. The resulting granular-ball representation is then used for downstream clustering.

When applied with AC and KMeans++ as downstream clustering procedures, MDL-GBG produces higher scores than the corresponding sample-level methods on most of the 18 datasets completed by all compared methods. Among the eight methods, MDL-GBG+AC obtains the lowest average rank for ARI, ACC, and NMI. The Nemenyi test identifies significant differences between MDL-GBG+AC and both AC and GBCT on all three metrics, as well as between MDL-GBG+AC and KMeans++ on NMI. The differences from KMeans++ on ARI and ACC, and from MDL-GBG+KM++, WGBC, MDMSC, and GBSC on all three metrics, remain below the critical difference.

The effects of core-residual refinement and residual reassignment are further distinguished by the ablation results. Removing $M_3$ reduces the average ARI from 0.5083 to 0.4657, whereas disabling residual reassignment produces a larger decrease to 0.3417. In the complete method, approximately 92.4\% of the samples separated during core-residual refinement are subsequently reassigned to stable balls. Thus, separating peripheral samples during generation does not necessarily exclude them from the final ball representation. The larger decline observed without reassignment is consistent with the substantial proportion of residual samples eventually incorporated into stable balls. The behavior of $M_3$ also depends on the shell-radius factor $\beta$. Increasing $\beta$ generally reduces the frequency of core-residual refinement, although the corresponding clustering performance does not change monotonically. Among the tested values, $\beta=2.00$ yields the highest average ARI, ACC, and NMI across the seven sensitivity datasets, while individual datasets favor different settings.

By processing local regions without constructing a full sample-to-sample matrix, MDL-GBG avoids the memory requirements associated with such a matrix. Under the 32~GB memory configuration, it completes CDC and Crop, whereas AC, MDMSC, and GBSC encounter MemoryError. This advantage in memory usage is accompanied by considerable computation on large datasets, as each local region must be evaluated under competing descriptions. The cost is particularly pronounced for $M_3$, which examines multiple residual sizes and evaluates their corresponding core descriptions and reference volumes. In the largest sample-size group, MDL-GBG+AC and MDL-GBG+KM++ require approximately 27,834 and 27,973 seconds on average, respectively.

The current formulation also limits the range of local structures considered during generation. The isotropic Gaussian description in $M_1$ and the first principal-component ordering used to construct $M_2$ candidates may be restrictive for regions with strong directional variation or more complex internal arrangements. Extending these descriptions to accommodate different local distributions and partition structures would allow additional representations to be compared under the MDL criterion. Similarly, the variation observed in the sensitivity analysis motivates further investigation of how the residual reference region could be determined from local data rather than through a common shell-radius factor. Such extensions would need to account for the complexity of additional candidate models, particularly given the computational cost already observed for core-residual refinement. Reusing geometric statistics across successive residual sizes could reduce repeated calculations in $M_3$ and make more flexible local comparisons computationally feasible. Beyond the generation stage, the number of granular balls is determined from the data, but the downstream clustering procedures still require a prescribed number of clusters $K_c$. Extending description-length-based model selection to the grouping of the generated balls would provide a possible means of determining the final cluster number without specifying it in advance.

\section*{CRediT authorship contribution statement}
\textbf{Zeqiang Xian}: Conceptualization, Methodology, Writing – original draft, Software, Validation.
\textbf{Caihui Liu}: Conceptualization, Methodology, Writing – review \& editing, Validation, Supervision. \textbf{Yong Zhang}: Software, Data curation. 
\textbf{Wenjing Qiu}: Software, Data curation. 
\textbf{Duoqian Miao}: Writing – review \& editing.
\textbf{Witold Pedrycz}: Writing – review \& editing.

\section*{Declaration of Competing Interest}
The authors declare that they have no known competing financial interests or personal relationships that could have appeared to influence the work reported in this paper.

\section*{Data availability}
Data will be made available on request.

\section*{Acknowledgment}
The research is supported by the National Natural Science Foundation of China under Grant No. 62566003, Graduate Innovation Funding Program of Jiangxi Province under Grant No. YC2025-S224.

\bibliographystyle{elsarticle-num}
\biboptions{sort&compress}
\bibliography{References}

\end{document}